\documentclass{article}

    \PassOptionsToPackage{numbers, compress}{natbib}
\usepackage[preprint]{neurips_2026}

\usepackage[utf8]{inputenc} 
\usepackage[T1]{fontenc}    
\usepackage{hyperref}       
\usepackage{url}            
\usepackage{booktabs}       
\usepackage{amsfonts}       
\usepackage{nicefrac}       
\usepackage{microtype}      
\usepackage{xcolor}         
\usepackage{subcaption}

\usepackage{blindtext}
\usepackage[normalem]{ulem}
\usepackage{cancel}
\usepackage{mathrsfs}

\usepackage[mathscr]{euscript}
\usepackage{amsthm}
\usepackage{amsmath}
\usepackage{amsfonts}
\usepackage{bbm}
\usepackage{xcolor}
\usepackage[numbers]{natbib}
\usepackage{mhequ}
\usepackage{relsize}
\usepackage{tikz}
\usepackage{mathtools}
\usetikzlibrary{arrows.meta, decorations.pathmorphing, positioning, decorations.markings}

\usepackage{xr}

\DeclareMathOperator*{\argmax}{arg\,max}
\DeclareMathOperator*{\argmin}{arg\,min}

\newcommand{\bR}{\mathbb{R}}
\newcommand{\bN}{\mathbb{N}}

\newcommand{\cT}{\mathcal{T}}

\newcommand{\cB}{\mathcal{B}}

\newcommand{\cL}{\mathcal{L}}

\newcommand{\sfE}{{\sf E}}

\newcommand{\sfP}{{\sf P}}

\newcommand{\rt}{\right}
\newcommand{\lt}{\left}

\makeatletter
\newcommand*{\indep}{%
  \mathbin{%
    \mathpalette{\@indep}{}%
  }%
}
\newcommand*{\nindep}{%
  \mathbin{
    \mathpalette{\@indep}{\not}
  }%
}
\newcommand*{\@indep}[2]{%
  \sbox0{$#1\perp\m@th$}
  \sbox2{$#1=$}
  \sbox4{$#1\vcenter{}$}
  \rlap{\copy0}
  \dimen@=\dimexpr\ht2-\ht4-.2pt\relax
  \kern\dimen@
  {#2}%
  \kern\dimen@
  \copy0 
} 
\makeatother

\newcommand{\ba}{\begin{align*}}
\newcommand{\ea}{\end{align*}}

\newcommand{\be}{\begin{equs}}
\newcommand{\ee}{\end{equs}}

\usepackage{amsthm}

\usepackage{amsthm}

\newtheorem{theorem}{Theorem}
\newtheorem{lemma}{Lemma}
\newtheorem{proposition}{Proposition}

\theoremstyle{definition}

\theoremstyle{remark}

\title{MARCEDES: Score-based causal discovery under non-Gaussianity with continuous optimization
}

\author{%
  Anamitra~Chaudhuri
  \thanks{Part of this work was conducted while the author was at The University of Texas at Austin.} \\
  Department of Statistics\\
  Florida State University\\
  Tallahassee, FL 32306 \\
  \texttt{achaudhuri@fsu.edu} \\
  \And
  Anirban Bhattacharya \\
  Department of Statistics \\
  Texas A\&M University \\
  College Station, TX 77843 \\
  \texttt{anirbanb@stat.tamu.edu} \\
  \AND
  Yang Ni \\
  Department of Statistics and Data Science \\
  The University of Texas at Austin \\
  Austin, TX 78705 \\
  \texttt{yang.ni@austin.utexas.edu} \\
}

\begin{document}

\maketitle


\begin{abstract}
We consider the problem of learning the underlying causal directed acyclic graph (DAG) structure corresponding to a structural equation model (SEM) with non-Gaussian errors. Motivated by an intentionally misspecified non-Gaussian SEM with all Laplace errors, we first introduce the mean absolute residual risk, defined over the space of all real matrices, and show that, asymptotically, the risk of the true weighted causal DAG matrix is strictly smaller than that of any other matrix.
Nevertheless, to enhance generality and account for high-dimensional and finite-sample settings, we further incorporate row-specific sparsity penalties along with a soft DAG constraint to derive a continuous score function over the space of real matrices. Accordingly, we propose a score-based DAG learning method, named MARCEDES, formulated as an unconstrained score minimization problem, which can be efficiently solved using gradient-based optimization techniques, thereby circumventing the challenges associated with constrained optimization.
Furthermore, we develop a computational algorithm to handle the non-smoothness of the score objective and to enable optimal tuning of row-specific sparsity penalties under a generalized Bayes framework. Finally, we demonstrate the efficiency and improved performance of the proposed method over existing approaches through an extensive simulation study.
\end{abstract}

\section{Introduction}

It is a fundamental problem to learn the underlying graphical structure, often encoding the underlying directed causal relationships in complex systems, from observational data arising in various domains such as public health \cite{shen2020challenges}, genomics \cite{sachs2005causal}, economics \cite{imbens2004nonparametric}, philosophy \cite{glymour2019review}, and artificial intelligence \cite{xia2021causal}. 
However, learning directed acyclic graphs (DAGs) from data is computationally challenging due to the super-exponential size of the DAG space \cite{andersson1997characterization}, the difficulty of enforcing the combinatorial constraint of acyclicity \cite{zheng2018dags}, and the fact that observational data generally identify DAGs only up to Markov equivalence classes \cite{heckerman1995learning}. Consequently, developing scalable and reliable methods for DAG learning remains a central challenge in modern machine learning and statistics.

Existing methods for estimating Markov equivalence classes are commonly grouped into constraint-based, score-based, and hybrid approaches \cite{drton2017structure}. Constraint-based methods, including PC \cite{spirtes2001causation}, FCI \cite{spirtes2001anytime}, RFCI \cite{colombo2012learning} etc. recover graph structure by testing conditional independence relations, while score-based methods instead optimize a scoring criterion over DAGs or their equivalence classes, with GES \cite{chickering2002optimal} being a prominent example. 
While these methods and their many variants have been extensively developed for Gaussian DAG models \cite{giudici2003improving, grzegorczyk2008improving, goudie2016gibbs, kuipers2017partition, kalisch2007estimating, maathuis2009estimating}, they generally target recovery of the Markov equivalence class. 
In contrast, non-Gaussian DAG models can be exactly identifiable under suitable assumptions \cite{shimizu2006linear}, but the corresponding methodological literature is comparatively limited. Existing approaches are often based on ICA \cite{comon1994independent} or rely on causal ordering estimation \cite{tashiro2014parcelingam, hyvarinen2013pairwise, wang2020high} and independence testing \cite{shimizu2011directlingam, zhao2022learning}, making their accuracy sensitive to intermediate estimation and testing steps. Despite showing promising empirical performance \cite{hoyer2009bayesian, shimizu2014bayesian, chaudhuri2025consistent2}, score-based methods are especially scarce in this setting because likelihood-based scores require choosing a tractable non-Gaussian working model, and any such choice inevitably introduces model misspecification when the true error distribution is unknown; thus, the score must be carefully designed to both exploit non-Gaussianity for exact DAG recovery and remain computationally amenable.


Beyond this, existing score-based approaches for non-Gaussian DAG learning \cite{hoyer2009bayesian, shimizu2014bayesian, chang2024order, chaudhuri2025consistent} also face the inevitable NP-hard problem \cite{chickering1996learning} of optimizing over the discrete, complex space of DAGs. A major advance is the NOTEARS \cite{zheng2018dags}, which reformulates acyclicity as a smooth equality constraint over real-valued matrices, enabling gradient-based optimization. However, such constrained formulations can introduce computational bottlenecks, including sensitivity to augmented Lagrangian tuning and numerical instability as the penalty coefficient grows to enforce acyclicity \cite{ng2022convergence, ng2020role}.

\paragraph{Our contributions.} We address these challenges through the following contributions.

\begin{itemize}
    \item We introduce the mean absolute residual risk, a criterion over real-valued matrices motivated by an intentionally misspecified Laplace error-SEM that naturally connects to Bayesian hierarchical modeling. We show that, under a broad class of non-Gaussian errors given by scale mixtures of Gaussians, the true causal weighted matrix asymptotically achieves strictly smaller risk than any other weighted causal matrix, whether acyclic or cyclic, almost surely.
    \item To improve robustness, finite-sample performance, and high-dimensional structure recovery, we add sparsity and soft acyclicity penalties \cite{ng2020role} to the risk. However, unlike GOLEM \cite{ng2020role}, which uses a single sparsity parameter, we assign equation-specific sparsity penalties to allow more flexible regularization across structural equations. This yields our proposed unconstrained score-minimization framework, named MARCEDES, for estimating the underlying causal DAG, avoiding the hard acyclicity constraint used in \cite{zheng2018dags}.
    \item Although the resulting optimization is unconstrained, the score objective remains non-convex and non-smooth. We address this by first obtaining a smooth surrogate \cite{mairal2014sparse, daubechies2010iteratively}, and then interpret the optimization as maximum-a-posteriori (MAP)-type estimation under a generalized Bayes framework \cite{bissiri2016general}. We then develop an alternating minimization algorithm with a principled cross-validation framework for tuning the prior hyperparameters, and similar ideas can be used more broadly for Bayesian MAP estimation \cite{bishop2006pattern}.
    \item Numerical experiments show that the proposed method improves both structure learning and parameter estimation over existing methods through various evaluation metrics.
\end{itemize} 







\section{Problem formulation}
\label{sec:prob_form}

\subsection{Structural equation model}
Consider $p$ random variables $X_j, j \in [p]$, which are generated by a linear recursive SEM given by, for every $j \in [p]$,
\begin{equation}\label{eq:sem}
X_j = \cB_j^TX + \epsilon_j \quad \text{with} \quad \epsilon_j \overset{\rm ind}{\sim} \sfP_j, 
\end{equation}
where the SEM coefficient vector $\cB_j = (\cB_{j1}, \dots, \cB_{jp})^T \in \bR^p$ quantifies all direct linear causal effects on $X_j$, and $\epsilon_j$ is an independent random noise following some unknown distribution $\sfP_j$. Equivalently, letting $X = (X_1, \dots, X_p)$, the above can be expressed as $X = \cB X + \epsilon$, where $\cB = (\cB_1, \dots, \cB_p)^T \in \bR^{p \times p}$ denotes the SEM coefficient matrix, and $\epsilon = (\epsilon_1, \dots, \epsilon_p)^T \in \bR^p$ is the random vector consisting of the independent noise variables. Moreover, we consider $n$ independent and identically distributed (iid) observations of $X$ following the model \eqref{eq:sem}, which are 
collected in the dataset $D_n = \{X^{(i)} : i \in [n]\}$.

Define a function $\gamma (\cdot) : \bR^{p \times p} \to \{0, 1\}^{p \times p}$ such that for any $B \in \bR^{p \times p}$, and for every $j, k \in [p]$, $(\gamma(B))_{jk} = 1$ if and only if $B_{jk} \neq 0$,  then 
$\gamma(\cB)$ represents the adjacency matrix of a DAG $([p], E_\cB)$, where the set of nodes $[p]$ represents the random variables and the set of edges $E_\cB$ satisfies that $(k, j) \in E_\cB$ if and only if there exists a direct linear causal effect of $X_k$ on $X_j$, i.e., $\cB_{jk} \neq 0$. Therefore, with a mild abuse of notations, we equivalently specify $\gamma(\cB)$ as the true underlying {\it causal} DAG, and in the same spirit, indicate $\cB$ also as the true {\it weighted causal} DAG. 

Due to the independence, the joint probability distribution of the noise variables is given by $\sfP = \otimes_{j \in [p]} \sfP_j$. Moreover, we denote by $\sfP_X$ the joint probability distribution $X$, induced by $\sfP$ through the model \eqref{eq:sem}.
In order to infer about $\sfP_X$, it is of significant interest to learn the true underlying causal DAG $\gamma(\cB)$, which additionally contains all conditional independence relationships encoded in $\sfP_X$ \cite{spirtes2000causation}, or even more preferably, estimate the true weighted causal DAG $\cB$ from the dataset $D_n$.

\paragraph{Linear non-Gaussian acyclic model: LiNGAM.}
However, learning the true underlying causal DAG is challenging primarily due to identifiability. Specifically, there may exist multiple equivalent linear recursive SEMs of the form \eqref{eq:sem} that induce the same data-generating distribution $\sfP_X$, leading to multiple true candidate causal DAGs. For example, if $\sfP_j$, $j \in [p]$, are all Gaussian, then any DAG that is \emph{Markov equivalent} to $\gamma(\cB)$ can be regarded as a distributionally equivalent causal DAG \cite{geiger2002parameter}. In contrast, when all error variables are non-Gaussian, the seminal work of LiNGAM \cite{shimizu2006linear} shows that there exists no other equivalent linear recursive SEM of the form \eqref{eq:sem}; equivalently, $\gamma(\cB)$ is \emph{uniquely identifiable}. We adopt this non-Gaussian error assumption throughout the present work.

\subsection{Score-based Causal DAG learning}

To learn the weighted causal DAG $\cB$, we focus on score-based approaches, which have received growing attention over the past decades. Typically, such methods first define a data-dependent \emph{score function} over the space of all $p \times p$ real matrices and then minimize it subject to the constraint that the induced adjacency matrix represents a DAG. Formally, if the score function based on $D_n$ is denoted by $S_n(\cdot): \bR^{p \times p} \to \bR$, then one considers
\begin{equation*}
    \min_{B \in \bR^{p \times p}} \; S_n(B) 
    \qquad \text{subject to} \qquad 
    \gamma(B) \; \text{is a DAG}.
\end{equation*}
However, this problem is computationally challenging due to the discrete and combinatorial nature of the acyclicity constraint, as well as the super-exponentially large space of DAGs \cite{andersson1997characterization, chickering1996learning}.

\paragraph*{Continuous relaxation.}    

As first proposed in the pioneering work \cite{zheng2018dags}, a useful way to tackle this problem is to adopt the technique of continuous relaxation, that is, replacing the discrete, combinatorial constraint by a smooth constraint, which still enforces acyclicity \cite{bello2022dagma, yu2021dags, wei2020dags, lachapelle2019gradient}. Specifically, in \cite{zheng2018dags} it has been established that the function $h(\cdot) : \bR^{p \times p} \to \bR$, defined as
$h(B) = \text{tr}(e^{B \circ B}) - p$,
satisfies that $h(B) = 0$ if and only if $\gamma(B)$ is a DAG. This consequently transforms the above problem into its equivalent form
\begin{equation}\label{eq:min_score}
    \min_{B \in \bR^{p \times p}} \; S_n(B) \qquad \text{subject to} \qquad h(B) = 0,
\end{equation}
which enjoys computationally efficient gradient-based continuous optimization techniques.

\section{Proposed method}

Although unique identifiability holds under general non-Gaussian errors, score-based DAG learning requires a specific choice of score function, which is naturally grounded in a {\it misspecified working model} for the errors \cite{hoyer2009bayesian, shimizu2014bayesian, chaudhuri2025consistent}. Therefore, an appropriate non-Gaussian distribution must be chosen to capture the non-Gaussianity present in the data while yielding a tractable framework for optimization.

\paragraph{Mean absolute residual risk.} 
Motivated by modeling the errors with Laplace distribution, or more specifically, considering 
the following misspecified {\it Laplace-error SEM}
\begin{equation}\label{eq:lap_sem}
\text{for every} \;\; j \in [p], \quad  X_j = B_j^TX + e_j \quad \text{with} \quad e_j \overset{\rm ind}{\sim} \text{Laplace}(\xi_j), \quad B \in \bR^{p \times p}, \;\; \xi \in \bR_+^p,
\end{equation}
as our working model fitted on $D_n$,
we consider the following risk function based on the mean absolute residuals when each variable is predicted upon the rest with $B$ as the coefficient matrix. Formally, we define $L_n(B)$ given by: if $\det(I - B) \neq 0$,
\begin{equation}\label{eq:L_n}
    L_n(B) = - \, \log|\det(I - B)| +  \,\sum_{j \in [p]} \log\Bigg(\frac{1}{n}\sum_{i \in [n]}|X_{j}^{(i)} - B_j^TX^{(i)}|\Bigg),
\end{equation}
and $L_n(B) = +\infty$, otherwise.
We refer it as the \textit{Mean absolute residual risk}. 
The following result illustrates the above point in detail, and further establishes its connection with hierarchical Bayesian modeling.
\begin{proposition}\label{prop:Bayes_SEM}
    Consider \eqref{eq:lap_sem} to be fitted on $D_n$ with the likelihood function denoted by $\ell_n(B, \xi)$, 
    and assume that $\xi_j \overset{\rm iid}{\sim} \pi(\xi_j) \propto 1/\xi_j$, then the maximized and marginalized likelihood are respectively
    \begin{equation*}
\max_{\xi} \; \log \ell_n(D_n | B, \xi) = c_n - n \, L_n(B), \quad \& \quad \int_\xi \ell_n(D_n | B, \xi) \prod_{j \in [p]}\pi(\xi_j) d\xi_j \propto \exp(- n \, L_n(B)),
\end{equation*}
where $c_n$ is some constant independent of $D_n$.
\end{proposition}


The proof can be found in Appendix \ref{app:pf_prop1}. While our problem formulation allows for general non-Gaussian error distributions, a particular class that we consider as a natural and practically relevant specialization is that of scale mixtures of Gaussian distributions. This class provides a flexible semiparametric family that can capture a wide range of non-Gaussian behaviors while retaining an interpretable latent Gaussian structure. In particular, for each node \(j \in [p]\), such a representation assumes that
\begin{align}\label{eq:smg_error}
\epsilon_j \mid \sigma_j \sim {\text N}(0, \sigma_j^2), \qquad \text{with} \qquad  \sigma_j \stackrel{\mathrm{ind}}{\sim} \mathsf{Q}_j,
\end{align}
where \(\mathsf{Q}_j\) is some unknown non-degenerate probability distribution supported on \((0,\infty)\). 
This class is widely regarded as a natural and expressive choice for modeling error distributions \cite{andrews1974scale, west1984outlier, west1987scale, box2011bayesian}. It preserves desirable structural properties such as symmetry and unimodality, and encompasses a broad range of distributions including Laplace, Student’s \(t\), Cauchy, and more generally the symmetric stable and exponential power families, as well as their mixtures, and polynomial-tailed distributions.
We emphasize that the above representation is not imposed in our general setup; rather, it serves us as a guiding and practically relevant subclass of non-Gaussian distributions, as we illustrate it in the following result.
\begin{theorem}[pairwise consistency]\label{thm:pair_const}
     Suppose that the errors are some (non-degenerate) scale mixture of Gaussian, that is, \eqref{eq:smg_error} holds, and $\sfE[|\epsilon_j|] < \infty$ for every $j \in [p]$. Fix any arbitrary $B \in \bR^{p \times p}, B \neq \cB$. 
     Then 
     we have 
    \[
    (L_n(B) - L_n(\cB)) \to \delta(B, \cB), \qquad \text{almost surely}, 
    \]
    where $\delta(B, \cB) > 0$, and is given as follows. If $\det(I-B) \neq 0$,
    \[\delta(B, \cB) = - \log |\det(W)| + \log \Big(\prod_{j \in [p]}\sfE[|W_j\epsilon|] \Big/\sfE[|\epsilon_j|] \Big) > 0, 
    \]
    with $W = (I - B)(I - \cB)^{-1}$, and $\delta(B, \cB) = +\infty$, otherwise.
\end{theorem}
The proof 
can be found in Appendix \ref{app:pf_thm1}.
The above result motivates using $L_n(B)$ as a score function, since it is asymptotically minimized, in a pairwise sense, at the true weighted causal DAG $\cB$ under a broad semiparametric family of non-Gaussian error distributions.

\paragraph{Proposed method: MARCEDES.}
In light of the preceding discussion, a natural approach is to use $L_n(B)$ as the score function in the score-based formulation. In fact, when the true errors belong to the scale-mixture-of-Gaussians family, the pairwise risk separation suggests that the acyclicity constraint may be relaxed asymptotically. Nevertheless, to accommodate more general non-Gaussian errors and improve finite-sample performance, we retain the DAG structure through a soft acyclicity penalty and further incorporate sparsity regularization. This leads to the following \emph{unconstrained} score-minimization problem, which we call \emph{Mean Absolute Residual risk based Continuous optimizEr for DirEcted acyclic graph Selection} (MARCEDES):
\begin{equation}\label{eq:min_score_unc}
    \min_{B \in \bR^{p \times p}}
    \; L_n(B) + \sum_{j \in [p]} \lambda_j \|B_j\|_1 + \lambda_D h(B),
\end{equation}
where $\lambda_D>0$ controls the strength of the DAG penalty, and $\lambda_j>0$, $j\in[p]$, are equation-specific sparsity parameters.
The use of $h(B)$ as a penalty rather than as a hard equality constraint corresponds to the \emph{soft} DAG constraint studied in \cite{ng2020role}. In the context of GOLEM \cite{ng2020role}, such a soft treatment of acyclicity, combined with sparsity regularization, was shown to offer computational advantages over the hard DAG constraint of \cite{zheng2018dags} and to improve empirical performance in several settings, including high-dimensional regimes. Our formulation differs from \cite{ng2020role} in two important ways. First, whereas GOLEM is based on a Gaussian likelihood and is therefore primarily tailored to Gaussian DAG models, MARCEDES is built on the mean absolute residual risk, which is designed to exploit non-Gaussianity for exact DAG identification. Second, instead of using a single global sparsity parameter $\lambda\|B\|_1$, we allow equation-specific penalties $\sum_{j\in[p]}\lambda_j\|B_j\|_1$, providing greater flexibility across structural equations. This adaptive regularization, inspired by related ideas in \cite{zou2006adaptive, meier2008group}, is particularly useful in high-dimensional settings and leads to improved estimation accuracy in numerical experiments.




\section{Optimization}
\label{sec:optim}

In this section, we develop an algorithm for solving the optimization problem of MARCEDES formulated in \eqref{eq:min_score_unc}. The objective is highly nonconvex and, even without the sparsity and DAG penalties, remains nonsmooth due to the absolute residuals in $L_n(\cdot)$. In addition, the method requires tuning several penalty parameters.

\subsection{Generalized Bayes framework}

We have $p$ equation-specific sparsity parameters $\lambda_j$, $j \in [p]$, collected as $\lambda=(\lambda_1,\ldots,\lambda_p)$, which must be properly tuned for accurate estimation. When $p$ is large, direct tuning over a $p$-dimensional grid becomes computationally prohibitive. To address this issue, we adopt a generalized Bayes approach \cite{bissiri2016general}. Specifically, motivated by Proposition \ref{prop:Bayes_SEM}, we consider the hierarchical Bayesian formulation therein and place a hyperprior on the sparsity parameters. Let $\pi_\theta(\cdot)$ be a prior distribution supported on $\bR_+$ and indexed by a lower-dimensional hyperparameter $\theta \in \Omega \subseteq \bR^d$, where typically $d \ll p$. We consider the prior structure
\begin{align}
\begin{split}
    B \mid \lambda_j, \, j \in [p]
    &\sim \pi(B \mid \lambda, \lambda_D)
    \propto
    \exp\Bigg(
        -\sum_{j \in [p]} \lambda_j \|B_j\|_1
        - \lambda_D h(B)
    \Bigg),\\
    \lambda_j
    &\overset{\rm iid}{\sim} \pi_\theta(\lambda_j),
    \qquad j \in [p].
\end{split}
\end{align}
Here, the prior on $B$ is understood as a generalized prior, where the row-wise $\ell_1$ term encourages sparsity across structural equations and $h(B)$, as in \eqref{eq:min_score}, softly penalizes deviations from acyclicity.
Under this formulation, the maximum-a-posteriori estimator, after the necessary rescaling and transformation, is obtained by solving
\begin{equation}\label{eq:MAP}
    \min_{B \in \bR^{p \times p}, \, \lambda\in\bR_+^p}
    \; L_n(B)
    + \sum_{j \in [p]} \lambda_j \|B_j\|_1
    + \lambda_D h(B)
    - \sum_{j \in [p]} \log \pi_\theta(\lambda_j).
\end{equation}
This formulation reduces the burden of selecting $p$ sparsity parameters separately by modeling them through a lower-dimensional hyperparameter $\theta$. Let the minimizer with respect to $B$ in \eqref{eq:MAP} be denoted by $\hat{B}(\theta,\lambda_D)$. In the remainder of this section, we develop a computational scheme for solving \eqref{eq:MAP}.




\subsection{Gradual enforcement of the DAGness penalty}\label{subsec:incr_DAG_pen}

The parameter $\lambda_D$ controls the strength of the DAGness penalty. Since the objective is highly nonconvex and may contain many closely spaced local optima, using a large value of $\lambda_D$ from the beginning can overly restrict the search to a neighborhood of the DAG chosen as the initialization point in our algorithm. To encourage broader exploration, including cyclic directed graphs and their nearby DAGs, we gradually increase the DAG penalty during optimization.

Specifically, let $\Lambda_D$ be an increasing grid of DAG-penalty values whose largest element is the target value $\lambda_D$. Denote its length by $T=|\Lambda_D|$, and write its elements as
$\lambda_{D(1)} < \cdots < \lambda_{D(T)}=\lambda_D$.
At the first stage, $t=1$, we compute $\hat{B}(\theta,\lambda_{D(1)})$ using the optimization scheme described below in Section \ref{subsec:comp_alg}, initialized at a chosen matrix $B^{(0)}$. For each subsequent stage $t=2,\ldots,T$, we compute $\hat{B}(\theta,\lambda_{D(t)})$ using the same scheme, initialized at the previous solution, $B^{(0)} = \hat{B}(\theta,\lambda_{D(t-1)})$.
After completing all stages, the final output is
    $\hat{B}(\theta,\lambda_D)
    =
    \hat{B}(\theta,\lambda_{D(T)}).$

\subsection{Core optimization algorithm}\label{subsec:comp_alg}

From \eqref{eq:L_n}, the risk $L_n(B)$ is nonsmooth because its second term involves absolute residuals, which prevents the direct application of standard gradient-based methods. To address this issue, we use a variational formulation of the $\ell_1$ norm, commonly used in iterative reweighting schemes for sparse optimization \cite{daubechies2010iteratively}, and discussed in \cite[Section~5.4]{mairal2014sparse}. Specifically, we use the identity $2|x|=\min_{\eta\in\bR_+}(x^2/\eta+\eta)$, with minimizer $\eta=|x|$. Thus, letting $\eta=(\eta_{ij}: i\in[n],\,j\in[p])\in\bR_+^{np}$ and defining
\[
\cL_n(B,\eta)
:=
-\log|\det(I-B)|
+
\sum_{j=1}^p
\log\Bigg(
\frac{1}{2n}
\sum_{i=1}^n
\Bigg(
\frac{(X^{(i)}_j-B_j^T X^{(i)})^2}{\eta_{ij}}
+
\eta_{ij}
\Bigg)
\Bigg),
\]
we have $L_n(B)=\min_{\eta\in\bR_+^{np}}\cL_n(B,\eta)$.
Therefore, for a given DAG-penalty value $\lambda_{D(t)}\in\Lambda_D$, the minimization problem in \eqref{eq:MAP} can be equivalently written as
\begin{equation}\label{eq:eta_MAP}
    \min_{B \in \bR^{p \times p},\, \lambda\in\bR_+^p,\, \eta \in \bR_+^{np}}
    \;
    \cL_n(B,\eta)
    + \lambda_{D(t)} h(B)
    + \sum_{j \in [p]} \lambda_j \|B_j\|_1
    - \sum_{j \in [p]} \log \pi_\theta(\lambda_j).
\end{equation}
We solve \eqref{eq:eta_MAP} by alternating minimization over $\eta$, $\lambda$, and $B$ until convergence or until a prescribed maximum number of iterations is reached. The algorithm is described formally below.

    \paragraph{Initialization.}
If $t=1$, we initialize the algorithm at an initial point $B^{(0)}$, such as the empty DAG, i.e., $B^{(0)}=\boldsymbol{0}$, or another reliable estimate when available. We note that using an informative initialization can substantially improve performance relative to the initial estimator itself, provided such an estimate is readily obtainable. Further details on the initialization choices used in our experiments are provided in Appendix~\ref{app:supp_numexp}.
If $t > 1$, set
$B^{(0)}=\hat{B}(\theta,\lambda_{D(t-1)})$.

Then, at each iteration $k\in\bN$, perform the following steps.

\paragraph{Minimization with respect to $\eta$.}
Using the minimization identity above, update
\[
    \eta_{ij}^{(k)}
    =
    \left|X_{j}^{(i)} - B_{j}^{(k-1)T}X^{(i)}\right|,
    \qquad
    i\in[n],\; j\in[p].
\]

\paragraph{Minimization with respect to $\lambda$.}
Next, for every $j\in[p]$, update
\[
    \lambda_j^{(k)}
    =
    \argmin_{\lambda_j\in\bR_+}
         \; -\log \pi_\theta(\lambda_j)
        +
        \lambda_j \|B_j^{(k-1)}\|_1.
\]

    \paragraph*{Minimization with respect to $B$.}
This step is more involved and considers the following minimization with the
previously updated values of $\eta$ and $\lambda$:
\[
\argmin_{B \in \bR^{p \times p}} \;
\cL_n(B,\eta^{(k)})+\lambda_{D(t)}h(B)
+\sum_{j\in[p]}\lambda_j^{(k)}\|B_j\|_1.
\]
Since both $\cL_n(B,\eta^{(k)})$ and $h(B)$ are smooth, we first take an
Adam \cite{kingma2014adam} step on the smooth component
and then apply a row-wise ISTA \cite{beck2009fast} shrinkage step induced by the $\ell_1$
penalty, and iterate over $M_{\rm in}$ many inner iterations. Formally, let the smooth component be
\[
G(B;\eta^{(k)}) = \cL_n(B,\eta^{(k)})+\lambda_{D(t)}h(B),
 \qquad \text{and} \qquad g^{(s)}=\nabla_B G(\widetilde{B}^{(s)};\eta^{(k)})
\]
be the gradient computed on a minibatch at inner iteration $s \in [M_{\rm in}]$, where the initialization $\widetilde{B}^{(1)} = B^{(k-1)}$. Then, for some pre-specified $\beta_1, \beta_2 \in (0, 1)$, and the moment parameters $m^{(0)}, v^{(0)}$, the Adam
moments are updated as
\[
m^{(s)}=\beta_1m^{(s-1)}+(1-\beta_1)g^{(s)},\qquad
v^{(s)}=\beta_2v^{(s-1)}+(1-\beta_2)(g^{(s)}\circ g^{(s)}),
\]
followed with bias-corrections $\widehat m^{(s)}={m^{(s)}}/({1-\beta_1^s})$ \, and \, $
\widehat v^{(s)}={v^{(s)}}/({1-\beta_2^s})$.
Subsequently, the Adam descent step is, define $\widehat B^{(s)}$ such that for every $j, \ell \in [p]$, with some learning rate $\ell_{\rm Adam} > 0$,
\[
\widehat B^{(s)}_{j\ell}
=
\widetilde B^{(s)}_{j\ell} -\ell_{\rm Adam}
\frac{\widehat m^{(s)}_{j\ell}}{\sqrt{\widehat v^{(s)}_{j\ell}}+\varepsilon_{\rm Adam}}.
\]
Then, we apply row-specific soft-thresholding operator on $\widehat B^{(s)}$ to obtain $\widetilde B^{(s+1)}$, i.e.,
\[
\widetilde B^{(s+1)}_{j\ell}
=
\operatorname{sign}(\widehat B^{(s)}_{j\ell})
\left(|\widehat B^{(s)}_{j\ell}|-\ell_{\rm ISTA}\lambda_j^{(k)}\right)_+,
\qquad j, \ell\in[p].
\]
Finally, we set the diagonal entries exactly to zero, i.e., for every $j\in[p]$,
$\widetilde B^{(s+1)}_{jj}=0$.
After $M_{\rm in}$ such inner updates, we set
$B^{(k)}=\widetilde B^{(M_{\rm in})}$.
Typically, $M_{\rm in}$ is chosen to be small, as it is unnecessary to solve this subproblem to high accuracy given that $\lambda^{(k)}$ and $\eta^{(k)}$ are only intermediate estimates.  
    
    \paragraph{Stopping criterion.} We stop if $||B^{(k)} - B^{(k-1)}||_F \leq \epsilon_B$, 
    or we reach $k = M$. Otherwise, we iterate by restarting the minimization step with respect to $\eta$ with the estimate $B^{(k)}$ and $k \gets k+1$.

\subsection{Cross-validation over the hyper-parameter}

To further improve estimation accuracy and reduce uncertainty in the choice of the prior hyperparameter $\theta$, we use $K$-fold cross-validation to select $\theta$ from a finite candidate set.

Let $I_v$, $v\in[K]$, be $K$ equally sized disjoint partitions of $[n]$, and let $\Theta\subseteq\Omega$ be a finite set of candidate values for $\theta$. For each $\theta\in\Theta$, let $\hat{B}^{[-v]}(\theta,\lambda_D)$ denote the minimizer with respect to $B$ obtained from the training samples $\{X^{(i)}: i\notin I_v\}$, namely,
\begin{equation}\label{eq:MAP_cross}
    \min_{B \in \bR^{p \times p}, \, \lambda\in\bR_+^p}
    \;
    L_n(B; I_v^c)
    + \sum_{j \in [p]} \lambda_j \|B_j\|_1
    + \lambda_D h(B)
    - \sum_{j \in [p]} \log \pi_\theta (\lambda_j),
\end{equation}
where, for any $s\subseteq[n]$, the empirical risk $L_n(B;s)$ is defined using the samples $\{X^{(i)}:i\in s\}$ as
\[
L_{n}(B; s)
=
-\log|\det(I - B)|
+
\sum_{j = 1}^p
\log\left(
\frac{1}{|s|}
\sum_{i \in s}
\left|X^{(i)}_{j} - B_j^T X^{(i)}\right|
\right).
\]
The cross-validated risk is then defined as
\[
    R(\theta)
    =
    \frac{1}{K}
    \sum_{v=1}^K
    L_n\!\left(\hat{B}^{[-v]}(\theta,\lambda_D); I_v\right).
\]
Finally, we select
$
    \hat{\theta}_{\rm CV}
    =
    \argmin_{\theta\in\Theta} R(\theta),
$
and compute the final estimator $\hat{B}(\hat{\theta}_{\rm CV},\lambda_D)$ using the optimization scheme described in Sections \ref{subsec:incr_DAG_pen} and \ref{subsec:comp_alg}.

\subsection{Empirical Bayes thresholding}



The estimate $\hat{B}(\theta,\lambda_D)$, for a fixed $\theta$, or its cross-validated version $\hat{B}(\hat{\theta}_{\rm CV},\lambda_D)$ may still contain cycles, primarily because some nonzero entries may be shrunk close to zero while effectively representing absent edges. Therefore, a final refinement step is needed to set small spurious coefficients exactly to zero and obtain a transparent estimate of the causal DAG.

Since the diagonal entries of $\hat{B}(\hat{\theta}_{\rm CV},\lambda_D)$ are already constrained to be zero by the algorithm, thresholding is applied only to the off-diagonal entries. Specifically, we perturb each off-diagonal entry by adding a small independent Gaussian noise variable from $\mathrm{N}(0,\sigma^2)$, for example with $\sigma=0.1$. We then apply empirical Bayes thresholding \cite{johnstone2004needles} to these noisy off-diagonal entries to estimate a threshold, below which coefficients are truncated to zero. This procedure is repeated for a sufficient number of replications, say 100, and the average estimated threshold is applied to $\hat{B}(\hat{\theta}_{\rm CV},\lambda_D)$.
If the resulting matrix still contains cycles, which is unlikely due to the DAG penalty and non-Gaussianity, we iteratively remove the edge with the smallest absolute weight until the estimated graph becomes acyclic, as is commonly done in continuous optimization approaches \cite{zheng2018dags, ng2020role}.

\section{Numerical Experiments}

In this section, we evaluate the empirical performance of the proposed method and compare it with existing benchmarks for learning linear DAGs with or without continuous optimization. Specifically, we consider PC \cite{spirtes2000causation}, GOLEM \cite{ng2020role}, ICA-LiNGAM \cite{shimizu2006linear}, DirectLiNGAM \cite{shimizu2011directlingam} and TL \cite{zhao2022learning} as our benchmarks.  Our primary focus is on structure recovery, while also assessing parameter estimation accuracy. 

\paragraph{Simulation setup.}

We consider data to be generated from the linear SEM, formalized as $X = \cB X + \epsilon$ according to \eqref{eq:sem},
where the underlying causal graph $\gamma(\cB)$ is generated using 
the standard graph model of
\emph{Erd\H{o}s--R\'enyi} (ER) \cite{erdos1959random}. 
Specifically, in this study, we consider ER$-2$ 
as our ground truth causal DAG, where by definition ER$-k$ graphs have $kp$ many expected edges, for any $k \in \bN$. Furthermore, for the generated causal graph, the edge weights are sampled from a Uniform distribution, ensuring various signal strength over the weighted causal graph $\cB$. 
Moreover, to incorporate non-Gaussianity, we consider three noise distributions: Uniform, Laplace, and Student's \(t\), where the errors are generated such that each distribution appears almost in equal proportion in the data generating process. Finally, we vary the sample size and dimensionality across $n \in \{50, 200\}$, and $p \in \{10, 20, 50, 100\}$, respectively, covering both low- and high-dimensional regimes, as well as small- and large-sample settings. More details about the simulation setup can be found in Appendix \ref{app:supp_numexp}.

\paragraph{Evaluation metrics.}
We assess performance using standard metrics that capture both structural and parametric accuracy. For structure learning, we consider the \emph{True Positive Rate} (TPR) and \emph{False Discovery Rate} (FDR) to quantify edge recovery, along with the normalized \emph{Structural Hamming Distance} (SHD) to measure overall graph discrepancy. We also report 
the \emph{Matthews Correlation Coefficient} (MCC), an overall summary metric suitable for evaluating graph learning.

Overall, smaller values of SHD, FDR, together with larger values of TPR and MCC, indicate better performance.

\subsection{Structure recovery}

\begin{figure}[t]
    \centering

    \begin{subfigure}{0.48\textwidth}
        \centering
        \includegraphics[width=\linewidth]{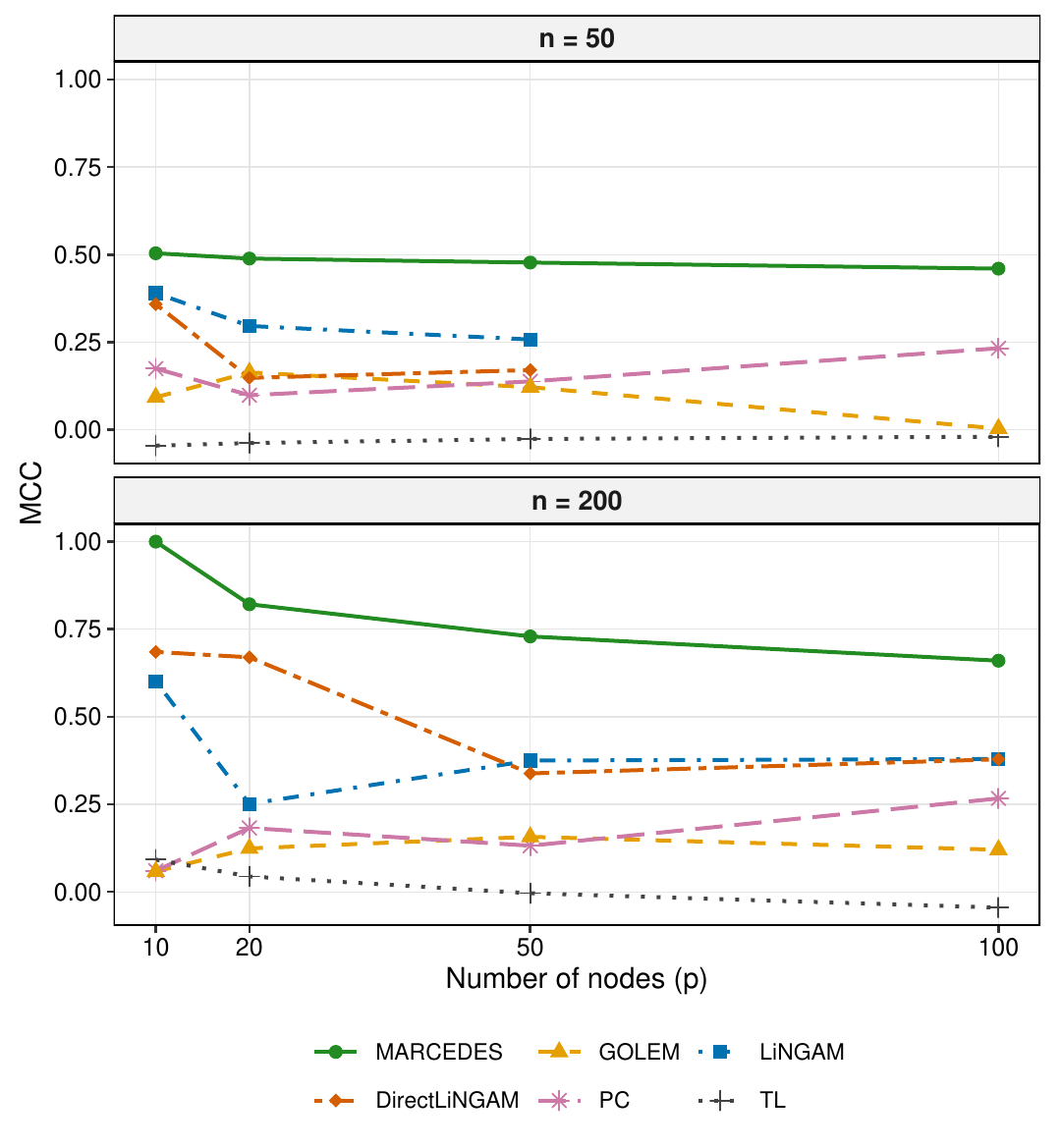}
        \caption{MCC}
        \label{fig:plotMCC}
    \end{subfigure}
    \hfill
    \begin{subfigure}{0.48\textwidth}
        \centering
        \includegraphics[width=\linewidth]{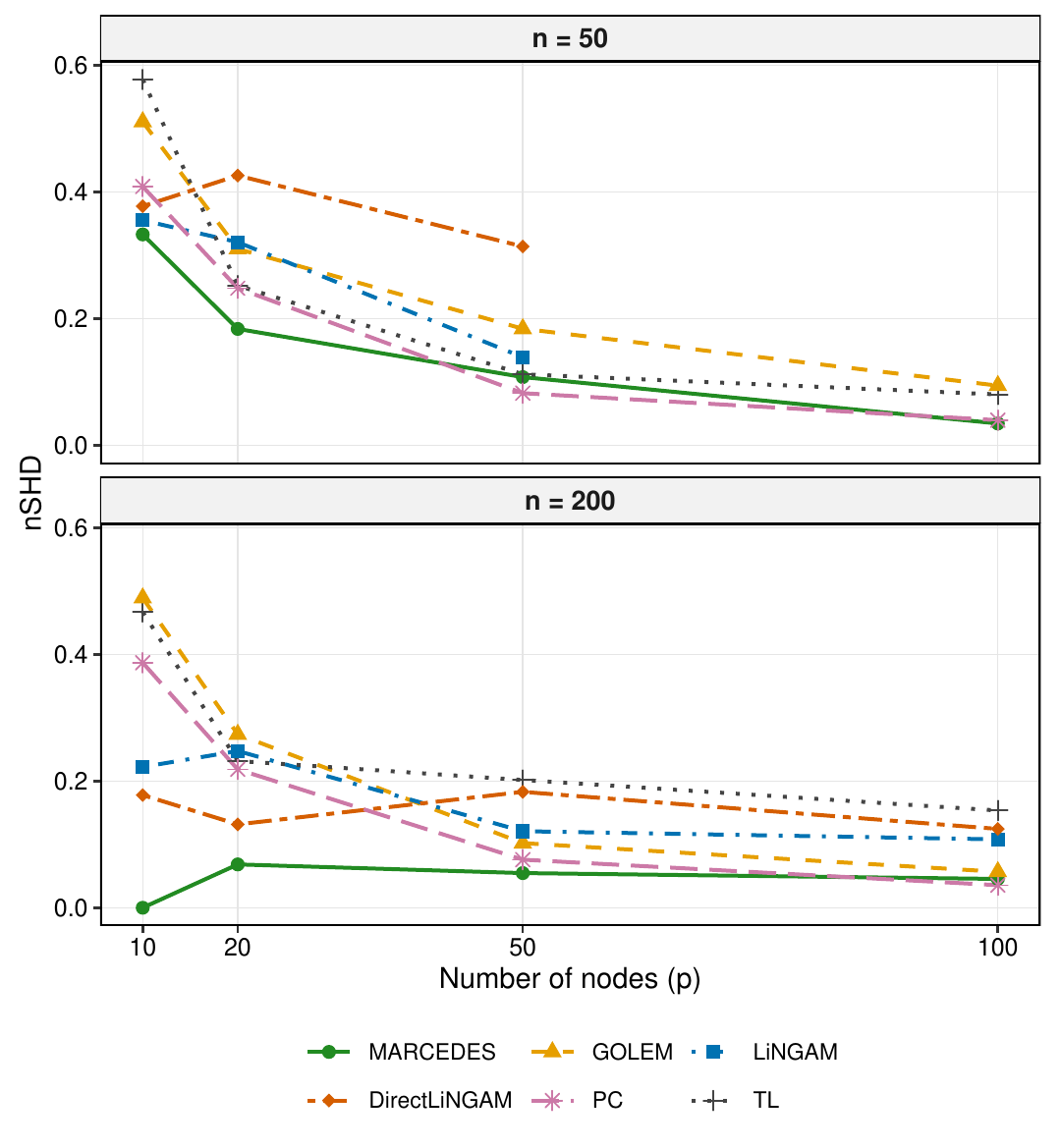}
        \caption{Normalized SHD}
        \label{fig:plotnSHD}
    \end{subfigure}

    \caption{Performance comparison of different in terms of overall structure recovery.}
    \label{fig:plot_struc}
\end{figure}

 We apply MARCEDES, together with several benchmark methods, to compare their performance in recovering the true underlying DAG. Figure~\ref{fig:plot_struc} reports the MCC and FDR values. Overall, MARCEDES yields a clear improvement in structure learning over the competing methods. As expected, GOLEM and PC, which are primarily tailored to Gaussian settings and do not explicitly exploit non-Gaussianity, achieve lower MCC values and higher FDR values compared with the non-Gaussian methods. Among the methods designed for non-Gaussian causal discovery, MARCEDES also demonstrates superior structure recovery.

In the small-sample setting, with $n=50$, MARCEDES shows substantial improvement in terms of both MCC and normalized SHD. For the larger-sample setting, with $n=200$, the performance improves further relative to the competing methods. In particular, when the sample size is large and the dimension is moderate, MARCEDES attains very small SHD values, indicating near-exact recovery of the underlying graph. 

It is important to note that, in higher-dimensional settings with $p>n$, existing LiNGAM-based methods sometimes fail to produce an estimate; consequently, some of their curves are incomplete in the figures. In contrast, MARCEDES remains applicable and is able to produce estimates in these regimes.

\subsection{Edge recovery}

\begin{figure}[t]
    \centering

    \begin{subfigure}{0.48\textwidth}
        \centering
        \includegraphics[width=\linewidth]{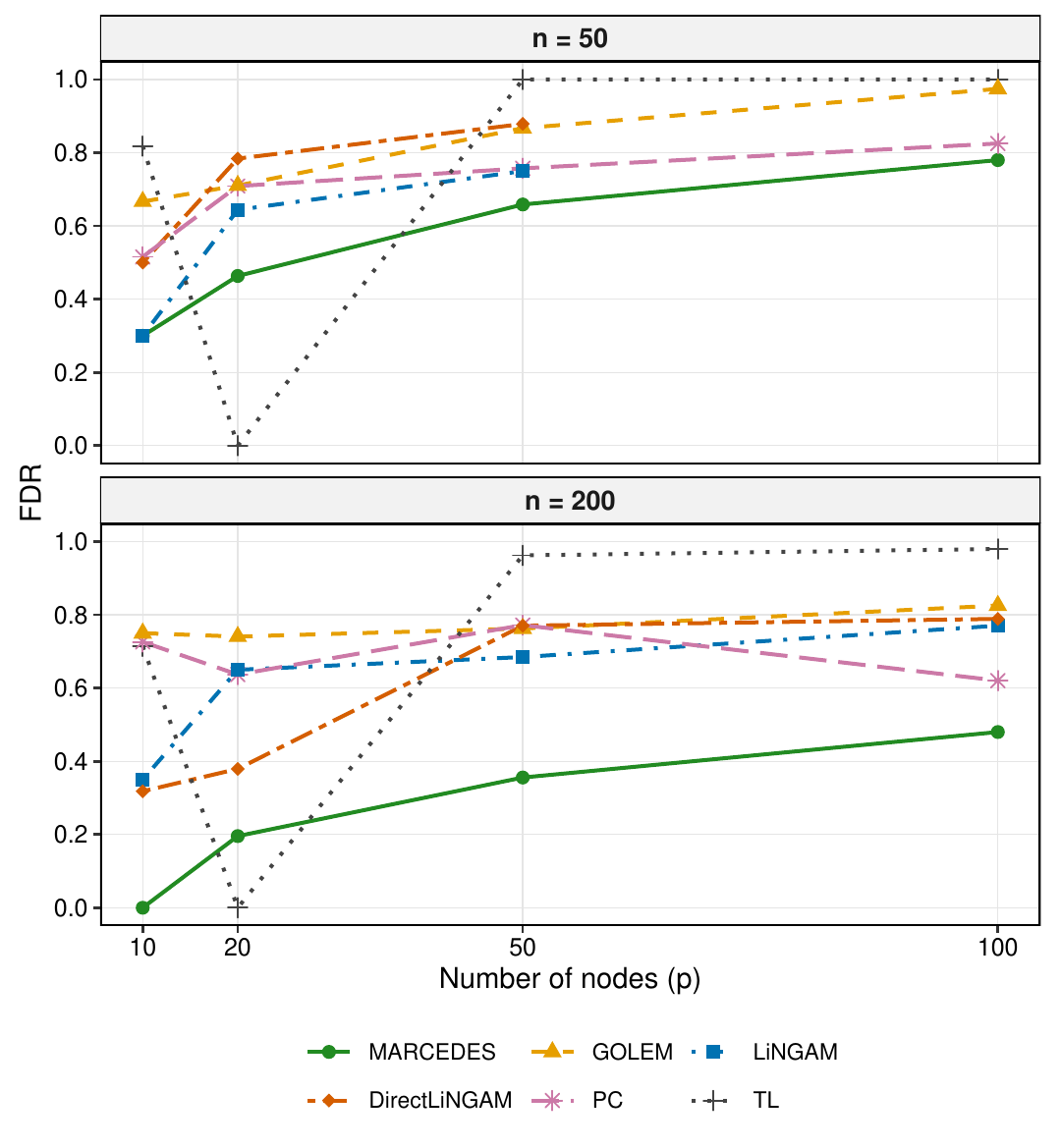}
        \caption{FDR}
        \label{fig:plotFDR}
    \end{subfigure}
    \hfill
    \begin{subfigure}{0.48\textwidth}
        \centering
        \includegraphics[width=\linewidth]{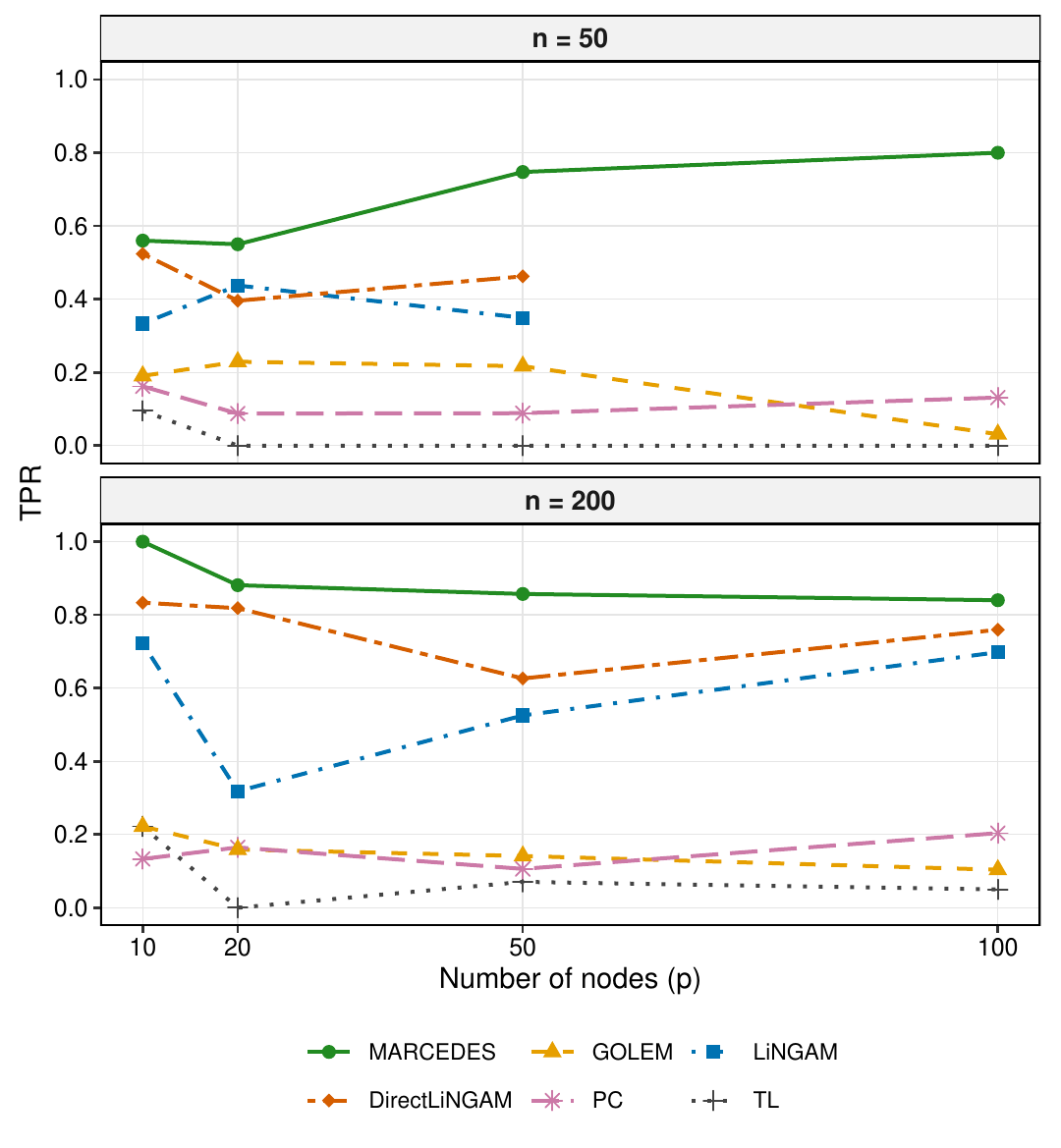}
        \caption{TPR}
        \label{fig:plotnTPR}
    \end{subfigure}

    \caption{Performance comparison of different in terms of edge recovery.}
    \label{fig:edge_recov}
\end{figure}

To assess edge-recovery performance at a more granular level, we consider FDR and TPR; see Figure~\ref{fig:edge_recov}. These metrics allow us to separately evaluate the extent to which a method avoids spurious edge selections and successfully identifies true causal edges. Across both small- and large-sample settings, the proposed method consistently outperforms the existing benchmarks. In particular, MARCEDES maintains a consistently high TPR across increasing dimensions, indicating stable recovery of true edges, while also achieving favorable FDR performance. This suggests that the proposed method provides more reliable edge recovery in both low- and high-dimensional regimes.

\section{Discussion}

We study causal DAG learning for linear SEMs with non-Gaussian errors and propose MARCEDES, a continuous score-based method built on the mean absolute residual risk. By combining row-specific sparsity penalties with a soft DAG constraint, MARCEDES leads to an unconstrained optimization framework that avoids the computational difficulties of enforcing hard acyclicity constraints. We further develop a gradient-based computational procedure to handle the non-smoothness of the objective, and our simulations show that MARCEDES can substantially improve structure learning performance over existing methods.

Several important directions remain for future work. While the present method is motivated by the mean absolute residual risk, it would be valuable to identify alternative, more general risk functions that enjoys favorable properties under broader classes of non-Gaussian errors and may further improve DAG learning performance. Another important direction is to establish formal statistical guarantees for the proposed estimator, including DAG selection consistency and parameter estimation consistency.

\paragraph{Broader impact.} The proposed method also has potential broader impact, since causal structure learning is useful in domains such as biology, economics, finance, healthcare, and the social sciences, where understanding directional relationships among variables is important. At the same time, estimated DAGs should be interpreted with caution. As with other causal discovery methods, MARCEDES may produce spurious or missing edges in finite samples, and its output may be affected by latent confounding, measurement error, selection bias, or violations of the underlying assumptions. Therefore, in decision-critical applications, learned structures should be validated by domain experts and supported by additional evidence before being used for consequential decisions.

\begin{ack}
A. Chaudhuri and Y. Ni were supported by NIH R01 GM148974.
Y. Ni was additionally supported by NSF DMS-2112943.
A. Bhattacharya was partially supported by NSF DMS-2210689 and NSF DMS-1916371.
The authors declare no competing interests.
\end{ack}


\newpage

{\small
\bibliographystyle{plainnat}
\bibliography{BCDnG-3}
}


\appendix

\section{Theoretical results}


\subsection{Proof of Proposition \ref{prop:Bayes_SEM}}\label{app:pf_prop1}
Following \eqref{eq:lap_sem}, we have $X = BX + e$, or equivalently, $X = (I - B)^{-1}e$. Furthermore, we have $e_j/\xi_j \overset{\rm iid}{\sim} \text{Laplace(1)}$, leading to the joint distribution of $X$ to be given by
\[
f_X(x) = |\det(I - B)| \prod_{j \in [p]} \frac{1}{2\xi_j}\exp\lt(-\frac{1}{\xi_j}|X_j - B_j^TX|\rt).
\]
Thus, the log-likelihood function is
\[
\log \ell_n(B, \xi) = -np \log 2 + n\log|\det(I - B)| - n \sum_{j \in [p]} \log \xi_j - \sum_{j \in [p]}\frac{1}{\xi_j}\sum_{i \in [n]} |X_j^{(i)} - B_j^TX^{(i)}|.
\]
Therefore, maximizing the above with respect to $\xi$, and letting $\hat{\xi} = \argmax_\xi \log \ell_n(B, \xi)$, we have
\[
\hat{\xi}_j = \frac{1}{n}\sum_{i \in [n]} |X_j^{(i)} - B_j^TX^{(i)}|, 
\]
which yields $\max_\xi \log \ell_n(B, \xi) = -np(1 + \log 2) - n L_n(B)$, proving the first part.

Furthermore, we have
\begin{align*}
    &\int_\xi \ell_n(B, \xi) \pi(\xi)d\xi\\
    &\propto\;\; 2^{-np} |\det(I - B)|^n \int_\xi \prod_{j \in [p]} \xi_j^{-(n+1)}\exp\Bigg(-\frac{1}{\xi_j}\sum_{i \in [n]}|X_j^{(i)} - B_j^TX^{(i)}|\Bigg)d\xi_j\\
    &\propto |\det(I - B)|^n \prod_{j \in [p]} \Bigg(\sum_{i \in [n]}|X_j^{(i)} - B_j^TX^{(i)}|\Bigg)^{-n} \propto \;\; \exp(-n \, L_n(B)).
\end{align*}
The proof is complete.

\subsection{Proof of Theorem \ref{thm:pair_const}}\label{app:pf_thm1}
Fix any arbitrary $B \in \bR^{p \times p}$. Then following \eqref{eq:sem} and \eqref{eq:lap_sem}, we have 
\[
e = (I - B)X = (I - B)(I - \cB)^{-1}\epsilon = W\epsilon, \quad \text{where} \;\; W = (I - B)(I - \cB)^{-1}.
\]
Thus, letting $W = (W_1, \dots, W_j)^T$, we have $e_j = W_j^T \epsilon$ for every $j \in [p]$.

Before proving Theorem \ref{thm:pair_const}, we first establish some important lemmas that will be useful later. Reiterating the assumption \eqref{eq:smg_error}, we have
\[\epsilon_j \mid \sigma_j \sim {\text N}(0, \sigma_j^2), \qquad \text{with} \qquad  \sigma_j \stackrel{\mathrm{ind}}{\sim} \mathsf{Q}_j,
\]
Let $\sigma = (\sigma_1. \dots, \sigma_p)^T$, and we define a random matrix $\Sigma$ whose rows are the transpose of $p$ independent random vectors $\sigma^{(i)} = (\sigma^{(i)}_1, \dots, \sigma^{(i)}_p)^T$, $i \in [p]$, which are identically distributed to $\sigma$, i.e., $\Sigma_{ij} = \sigma^{(i)}_j$. Then, 
We first establish the following results extending the result of \cite{chaudhuri2025consistent} to an arbitrary mixing matrix $W$ by adapting the arguments therein.
\begin{lemma}\label{lem:E_hadamard}
We have
    \[
    \sfE[\det(W \circ \Sigma)] = \det(W)\prod_{j \in [p]} \sfE[\sigma_j]
    \]
\end{lemma}
\begin{proof}
    Since \( W \circ \Sigma = \big((w_{ij}\sigma^{(i)}_j)\big) \), the determinant expansion gives us
\begin{align*}
\det(W \circ \Sigma)
= \sum_{\tau \in \cT_p} {\rm sgnt}(\tau)\; \prod_{i \in [p]} w_{i\tau(i)} \sigma^{(i)}_{\tau(i)},
\end{align*}
where $\cT_p$ denotes the set of all permutations $\tau(\cdot)$ of $[p]$.
Taking expectation and using independence,
\begin{align*}
\sfE[\det(W \circ \Sigma)]
&= \sum_{\tau \in \cT_p} {\rm sgnt}(\tau)\; \prod_{i \in [p]} w_{i\tau(i)} 
\prod_{i \in [p]} \sfE[\sigma_{\tau(i)}] \\
&= \Big(\prod_{i \in [p]} \sfE[\sigma_i]\Big)
\sum_{\tau \in \cT_p} {\rm sgnt}(\tau)\; \prod_{i \in [p]} w_{i\tau(i)}\\
&= \Big(\prod_{i \in [p]} \sfE[\sigma_i]\Big)\det(W),
\end{align*}
where ${\rm sgnt}(\cdot)$ denotes the signature of a permutation. This completes the proof.
\end{proof}

\begin{lemma}\label{lem:apply_Hadamard}
We have
    \[
    \prod_{j \in [p]}\sfE[\|W_j \circ \sigma\|_2] \geq  |\det(W)| \prod_{j \in [p]} \sfE[\sigma_j],
    \]
    where the equality holds if and only if
    $W = P\Delta$ for some permutation matrix $P$ and diagonal matrix $\Delta$ with all non-zero diagonal elements.
\end{lemma}
\begin{proof}
    We have
    \begin{align}\label{eq:ineq_core}
    \begin{split}
        |\det(W)| \, \prod_{j \in [p]} \sfE[\sigma_j]\ &= \big|\sfE[\det(W \circ \Sigma)]\big|\\
        &\leq \sfE[|\det(W \circ \Sigma)|]
        = \sfE[|\det(W \circ \Sigma)^T|]\\
        &\leq \sfE\Big[\prod_{j \in [p]}\|W_j \circ \sigma^{(j)}\|_2\Big]\\
        &= \prod_{j \in [p]}\sfE\Big[\|W_j \circ \sigma^{(j)}\|_2\Big] = \sfE\Big[\prod_{j \in [p]}\|W_j \circ \sigma\|_2\Big],
        \end{split}
    \end{align}
    where the first equality follows from Lemma \ref{lem:E_hadamard}, and the first inequality follows from the fact that 
    \begin{equation}\label{eq:applying_Hadamard}
    |\det(W \circ \Sigma)^T| \leq \prod_{j \in [p]}\|W_j \circ \sigma^{(j)}\|_2,
    \end{equation}
    which in turn holds due to
    {\it Hadamard's inequality} \cite{hadamard1893determinants}. Therefore, the equality holds throughout if and only if equality holds in both inequalities in \eqref{eq:ineq_core}. Specifically, in the second one, equality holds if and only if the independent random vectors $W_j \circ \sigma^{(j)}, j \in [p]$ are orthogonal almost surely, following \eqref{eq:applying_Hadamard} and the equality condition of the Hadamard's inequality \cite{hadamard1893determinants}.

    Fix any $i, j \in [p]$, then $W_i \circ \sigma^{(i)}$ and $W_j \circ \sigma^{(j)}$ are orthogonal almost surely when
\begin{align*}
\sum_{k \in [p]} W_{ik}W_{jk} \sigma^{(i)}_k\sigma^{(j)}_k \;\; \overset{\rm a.s.}{=} 0.
\end{align*}
However, since $\sigma_k^{(i)}, \sigma_k^{(j)}, k \in [p]$ are independent positive, and non-degenerate, the above holds if and only if $W_{ik}W_{jk} = 0$ for every $k \in [p]$. Since $W$ is also non-singular, this immediately implies that every column of $W$ has exactly one non-zero element, that is, $W = P\Delta$ for some permutation matrix $P$ and diagonal matrix $\Delta$. 

Furthermore, under the above condition of equality, we have $\det(W \circ \Sigma) = \det(\Delta) \prod_{j \in [p]} \sigma_j$. This subsequently yields 
\begin{align*}
|\sfE[\det(W \circ \Sigma)]| &= |\det(\Delta)| \prod_{j \in [p]} \sfE[\sigma_j]\\
&= \sfE\Big[|\det(\Delta)|\prod_{j \in [p]} \sigma_j\Big] = E[|\det(W \circ \Sigma)|],
\end{align*}
which, in fact, establishes the first equality in  \eqref{eq:ineq_core}.
The proof is complete.
\end{proof}

\begin{lemma}\label{lem:W=I}
    The equality condition in Lemma \ref{lem:apply_Hadamard} holds if and only if $e_j = \epsilon_j$ for every $j \in [p]$, or equivalently, $W = I$.
\end{lemma}
\begin{proof}
    We have $e = W\epsilon$, and $W = P\Delta$, which implies that $e_j, j \in [p]$ are independent. However, since $\cB$ is acyclic, following the exact identifiability of LiNGAM \cite{zheng2018dags, comon1994independent}, we must have $e_j = \epsilon_j$ for every $j \in [p]$, which further yields $W = I$.
\end{proof}

We are now ready to prove Theorem \ref{thm:pair_const}.

\begin{proof}[Proof of Theorem \ref{thm:pair_const}]
    Note that, for every $j \in [p]$, using the \textit{strong law of large numbers}, we have
    \[
    \frac{1}{n}\sum_{i \in [n]}|X_{j}^{(i)} - B_j^TX^{(i)}| \;\; \longrightarrow \;\; \sfE[|X_j - B_j^TX|] = \sfE[|e_j|] = \sfE[|W_j^T \epsilon|], 
    \]
    almost surely, since the corresponding first moment
    \[\sfE[|W_j^T \epsilon|] \leq \sum_{k \in [p]} |W_{jk}|\sfE[|\epsilon_k|] < \infty.\]
    In particular, when $B = \cB$, it is not difficult to note that this limiting value is $\sfE[|\epsilon_j|]$. Furthermore, due to acyclicity, following Lemma 1 in \cite{ng2020role}, we have $\log(|\det(I - \cB)|) = 0$. This, along with \eqref{eq:L_n} implies that we have, in almost sure sense,
    \[L_n(\cB) \;\; \longrightarrow \;\; \sum_{j \in [p]} \log(\sfE[|\epsilon_j|]) = \log \Big(\prod_{j \in [p]}\sfE[|\epsilon_j|]\Big) < \infty.
    \]
    In case $\det(I - B) = 0$, we have $L_n(B) = +\infty$, and thus, the result holds trivially. 
    
    Therefore, we assume that $\det(I - B) \neq 0$, and in that case, using \eqref{eq:L_n}, we have almost surely,
    \[
    L_n(B) \;\; \longrightarrow \;\; - \, \log |\det(I - B)| + \log \Big(\prod_{j \in [p]}\sfE[|e_j|]\Big).
    \]
    These lead us to have, almost surely,
    $
    L_n(B) - L_n(\cB) \; \to \; \delta(B, \cB), 
    $
    where we define
    \[\delta(B, \cB) = \log \Big(\prod_{j \in [p]}\sfE[|W_j\epsilon|] \Big/ |\det(I - B)|\Big) - \log \Big(\prod_{j \in [p]}\sfE[|\epsilon_j|]\Big). 
    \]
    Note that, $\det(W) = \det(I - B)$, since $\det(I - \cB) = 1$ again due to Lemma 1 in \cite{ng2020role}. Therefore, in order to prove that $\delta(B, \cB) \geq 0$, it suffices to show that
    \begin{equation}\label{ineq:to_show}
        \Big(\prod_{j \in [p]}\sfE[|W_j^T\epsilon|] \Big/ |\det(W)|\Big) \; \geq \; \prod_{j \in [p]}\sfE[|\epsilon_j|].
    \end{equation}
     Note that, due to \eqref{eq:smg_error}, we have $\sfE[|W_j^T\epsilon|] = \sqrt{2/\pi} \; \sfE\Big[\prod_{j \in [p]}\|W_j \circ \sigma\|_2\Big]$, and $\sfE[|\epsilon_j|] = \sqrt{2/\pi} \; \sfE[\sigma_j]$. Therefore, \eqref{ineq:to_show} reduces to showing that
     \[
    \prod_{j \in [p]}\sfE[\|W_j \circ \sigma\|_2] \geq  |\det(W)| \prod_{j \in [p]} \sfE[\sigma_j].
    \]
    Indeed, the above holds due to Lemma \ref{lem:apply_Hadamard}. Furthermore, again following Lemma \ref{lem:apply_Hadamard} and \eqref{lem:W=I}, the equality holds if and only if $W = I$, or equivalently $B = \cB$. The proof is complete.
\end{proof}

\section{Supplementary details on numerical experiments}\label{app:supp_numexp}

\subsection{Simulation setup}

First, we generate the true causal DAG from an ER-$2$ random graph model. The nonzero entries of the corresponding weighted adjacency matrix are then sampled independently from $\mathrm{Uniform}([-2,-0.5]\cup[0.5,2])$. Given this weighted DAG, we simulate data from a linear non-Gaussian SEM, where the error distributions are chosen from $\mathrm{Laplace}(0.8)$, Student's $t$ distribution with $5$ degrees of freedom, and $\mathrm{Uniform}[-\sqrt{3},\sqrt{3}]$, assigned in equal proportions across the structural equations. We then apply MARCEDES, along with the benchmark methods, to estimate the underlying weighted DAG matrix, and repeat the procedure over $30$ independent replications. Finally, we conduct the experiment across dimensions $p\in\{10,20,50,100\}$ and sample sizes $n\in\{50,200\}$ to evaluate performance under both low- and high-dimensional regimes with small and large samples. 

\subsection{Optimization method and implementation details}

In our implementation, we set the target DAG-penalty parameter to
$\lambda_D=50$ and use the increasing grid
\[
\Lambda_D=\{0,25,50\}.
\]
Thus, the acyclicity penalty is introduced gradually during optimization, rather
than being imposed at its full strength from the beginning. This continuation
strategy helps stabilize the optimization, especially in settings where the
initial estimate may be far from acyclic or where the sparsity pattern is still
being refined.

As an initial calibration step, we first run MARCEDES with a single global
sparsity penalty of the form $\lambda\|B\|_1$ and perform cross-validation over
a grid of $\lambda$ values to obtain a rough range for the sparsity level. This
preliminary experiment suggests that the optimal sparsity level lies in $(0,1)$
and is typically close to zero across our experiments, which is consistent with
the observations in \cite{ng2020role, zheng2018dags}. For brevity, we do not
report these preliminary calibration results. Motivated by this observation, we
choose the sparsity prior $\pi_\theta(\cdot)$ to be a
$\mathrm{Beta}(\alpha,\beta)$ distribution supported on $(0,1)$, which allows
the sparsity parameters to adapt flexibly while remaining within a practically
relevant range.

For the optimization procedure, it is generally beneficial to start from a
reasonably informative initial estimate. In our experiments, we choose
$B^{(0)}$ to be the estimate from the PC algorithm
\cite{spirtes2001causation} when $p \geq n$, and the estimate from
DirectLiNGAM \cite{shimizu2011directlingam} otherwise. This choice is motivated
by the fact that both methods are computationally efficient and can provide
useful initial graph estimates. In high-dimensional settings, where LiNGAM-based
methods may become unstable or fail to return an estimate, the PC algorithm
serves as a more robust initializer. When the sample size is sufficiently large
relative to the dimension, DirectLiNGAM provides a natural initialization that
directly exploits the non-Gaussian structure of the model. Alternatively, one
may initialize from the empty graph, i.e.,
$B^{(0)}=\boldsymbol{0}$.

For the minimization step with respect to $B$, we use the Adam--ISTA update
described in the previous section. The number of inner iterations is set to
$M_{\rm in}=1$. This is because the $B$-subproblem does not need to be solved to
high accuracy at each outer iteration while the auxiliary variables
$\eta^{(k)}$ and sparsity parameters $\lambda^{(k)}$ are still being updated.
Instead, a small number of inner updates is sufficient to make progress while
keeping the overall procedure computationally efficient. The Adam learning rate
is set to $\ell_{\rm Adam}=10^{-3}$, and the ISTA shrinkage step size is set to
$\ell_{\rm ISTA}=10^{-3}$. For Adam, we use the standard moment parameters
$\beta_1=0.9$, $\beta_2=0.999$, and
$\varepsilon_{\rm Adam}=10^{-8}$.

Finally, for the stopping criterion, we set $\epsilon_B=10^{-3}$ and allow a
maximum of $M=10^5$ outer iterations. The tolerance $\epsilon_B$ controls the
relative change in the estimated coefficient matrix across outer iterations,
while the maximum iteration limit serves as a safeguard against excessive
computation in difficult instances.

\paragraph{Cross-validation and post-processing.}
For cross-validation, we reparametrize the Beta hyperparameters in terms of the
corresponding mean and effective sample size:
\[
m=\frac{\alpha}{\alpha+\beta},
\qquad
\tau=\alpha+\beta.
\]
Here, $m$ controls the prior mean of the sparsity parameter, while $\tau$
controls the concentration of the Beta prior around this mean. We vary these
quantities over the grids
\[
m \in \{0.02,0.25,0.5,0.7\},
\qquad
\tau \in \{2,5,8,10\}.
\]
This parametrization provides a convenient and interpretable way to tune the
degree and strength of sparsity regularization.

Finally, we apply empirical Bayes thresholding
\cite{johnstone2005ebayesthresh} as a post-processing step. Specifically, after
obtaining the raw estimate $\hat{B}$ from the optimization procedure, we apply
empirical Bayes thresholding to its off-diagonal entries. To obtain a stable
threshold, we generate $100$ Gaussian perturbations of the off-diagonal
coefficient vector with standard deviation $0.1$ and apply empirical Bayes
thresholding to each perturbed vector. We use a Laplace prior together with the
median thresholding rule, and aggregate the resulting $100$ thresholds by taking
their mean. The final thresholded matrix is then obtained by retaining entries
whose absolute values exceed this aggregated threshold and setting all remaining
entries to zero. We also enforce a zero diagonal.

If the post-processed matrix is not acyclic, we apply a final cycle-removal step
by sequentially deleting the smallest-magnitude edge involved in a cycle until
the estimated graph becomes a DAG. This final step ensures that the reported
estimate satisfies the acyclicity constraint while making the smallest possible
changes, in magnitude, to the thresholded coefficient matrix.





\end{document}